\documentclass{article}

\PassOptionsToPackage{numbers,compress}{natbib}
\usepackage[sglblindworkshop, final]{neurips_2026}
\usepackage[utf8]{inputenc}
\usepackage[T1]{fontenc}
\usepackage{amsmath,amssymb,amsthm,booktabs,graphicx,microtype}
\usepackage[hidelinks]{hyperref}
\usepackage{subfig}
\usepackage{placeins}
\usepackage{wrapfig}
\usepackage{xcolor}
\definecolor{myblue}{HTML}{3585C9}
\newcommand{\blue}[1]{\textcolor{myblue}{#1}}
\usepackage{url}

\newcommand{\Prob}{\mathbb{P}}
\newcommand{\simplex}{\Delta^{C-1}}
\newcommand{\bbox}{\mathcal{C}^\mathrm{bbox}}
\newtheorem{lemma}{Lemma}
\title{
	When Has a Bayesian Neural Network Sampled Enough?
	Adaptive Inference Time with Statistical Guarantees
}
\workshoptitle{E-Values: From Statistics to ML}
\author{%
	Fabian Denoodt \\
	Eindhoven University of Technology \\
	\And
	Sibylle Hess \\
	Eindhoven University of Technology \\
}

\begin{document}
\maketitle
\begin{abstract}
Bayesian neural network predictions are commonly approximated using a fixed number of Monte Carlo samples per input, without controlling the resulting error that comes from this finite sample. We propose the use of confidence sequences to dynamically determine how many samples are needed while maintaining statistical guarantees. We consider several ways in which predictive probabilities are used, including identifying the most likely class, approximating the full predictive distribution, and resolving probability-threshold decisions. Sampling stops once the corresponding decision can be made with the desired guarantee. Experiments show that the method allocates the computational budget efficiently, assigning more samples to ambiguous inputs than to easy inputs while preserving reliable decisions and reducing overall latency relative to a fixed Monte Carlo budget.
\end{abstract}

\section{Introduction}
Bayesian neural networks (BNNs) provide a principled framework for representing predictive uncertainty, which is particularly valuable in high-stakes applications. They typically approximate the posterior predictive distribution $p(\mathbf{y}^{\star} \mid \mathbf{x}^{\star}, \mathbf{x}_{1:n}, \mathbf{y}_{1:n})$ by averaging predictions over posterior parameter samples $\mathbb{E}_{\boldsymbol{\theta} \sim q_\lambda}[p(\mathbf{y}^{\star} \mid \mathbf{x}^{\star}, \boldsymbol{\theta})]$ \citep{daxberger2021laplace,DBLP:conf/nips/MaddoxIGVW19,DBLP:conf/icml/GalG16,denoodt2025varbasedsmoothing,krause2025probabilisticAiBook}. In practice, this average is usually estimated using the same fixed number of stochastic forward passes for every input. Although simple, this ad hoc choice provides no guarantee that the Monte Carlo estimate is sufficiently accurate for the prediction. It may also waste computation on easy inputs while using too few samples for difficult ones.

We instead ask an input- and task-specific question: 
\begin{center}
\emph{When have enough posterior samples been drawn to make the required decision?} 
\end{center}
The answer depends on how the prediction will be used. If only the most likely label is needed, it suffices to separate the leading class from its competitors. If the full probability vector is required, every coordinate must be estimated accurately enough. Meanwhile, for threshold-based decisions, it is enough to determine whether any class probability exceeds a specified level $q$. Adapting the number of samples to these objectives would allow the method to concentrate computation on the inputs that need it most.

To make these adaptive decisions in a principled manner, we need uncertainty bounds that remain statistically valid as posterior samples accumulate. We therefore use the multivariate confidence sequence (CS) of \citet{denoodt2026tightness} to track the predicted probability vector and stop once the required decision is certified. Unlike repeatedly checking an ordinary confidence interval, CSs remain valid under repeated inspection and data-dependent stopping \citep{howard2021timeuniform}. This also distinguishes our approach from the approach of \citet{bethell2024robust}, which stops when changes in the estimated variance remain below a preset threshold and provides neither statistical guarantees nor an indication of how accurate the final prediction is.

Overall, we \textbf{\blue{(i)}} formulate anytime-valid stopping rules for various prediction use cases
 that control the Monte Carlo estimation error under repeated inspection and adaptive stopping, and \textbf{\blue{(ii)}}  show empirically that adaptive sampling uses fewer posterior draws for easy inputs while preserving reliable decisions, reducing overall inference time.

\section{Anytime-valid adaptive prediction}
\paragraph{Preliminaries.} Fix an input $\mathbf{x}$ and a classification problem with $C$ classes, indexed by $c\in[C]:=\{1,\ldots,C\}$. A BNN represents uncertainty about its parameters through an approximate posterior $q_\lambda(\boldsymbol\theta)$. Drawing $\boldsymbol\theta\sim q_\lambda$ and evaluating the network gives a class-probability vector $f_{\boldsymbol\theta}(\mathbf{x})$, and averaging these vectors gives the posterior predictive class probabilities
\begin{equation}
\boldsymbol\mu
:=
\mathbb{E}_{\boldsymbol\theta\sim q_\lambda}
\left[f_{\boldsymbol\theta}(\mathbf{x})\right]
\in\simplex, \label{eq:mu_predictive_posterior}
\end{equation}
where $\simplex=\{\mathbf m\in[0,1]^C:\sum_{c=1}^C m_c=1\}$ is the probability simplex. Each independent draw $\boldsymbol\theta_i\sim q_\lambda$ leads to a prediction $\mathbf y_i=f_{\boldsymbol\theta_i}(\mathbf{x})$, and after $t$ draws, the usual Monte Carlo estimate of $\boldsymbol\mu$ is $\widehat{\boldsymbol\mu}_t=t^{-1}\sum_{i=1}^t\mathbf y_i$.

A CS for $\boldsymbol \mu$ is a sequence of data-dependent sets $(\mathcal{C}_t)_{t\geq 1}$, where $\mathcal{C}_t=\mathcal{C}_t(\mathbf y_1,\ldots,\mathbf y_t)$ is constructed from the first $t$ predictions, such that
\begin{equation}
\mathbb{P}\!\left( \forall t \ge 1:\ \boldsymbol \mu \in \mathcal{C}_t \right) \ge 1 - \alpha,  \label{eq:cs-guarantee}
\end{equation}
for a prescribed error level $\alpha$~\citep{ryu_gamingbasedconfseq}.  Equation~\eqref{eq:cs-guarantee} immediately implies that, for any 
stopping time $\tau$ determined from the observed samples $\mathbf{y}_1,\dots,\mathbf{y}_\tau$, $\Prob(\boldsymbol\mu\in\mathcal C_\tau)\geq1-\alpha$. 

\paragraph{Stopping rules for various prediction tasks.} Suppose that the CS supplies coordinate-wise bounds $[L_{t,c},U_{t,c}]$ satisfying $\mathcal C_t\subseteq\prod_{c=1}^C[L_{t,c},U_{t,c}]$. We propose three task-specific stopping rules based on these bounds: one for selecting a class, one for approximating the full predictive distribution, and one for making a threshold decision. Proofs of all three guarantees are deferred to Appendix~\ref{app:stopping-proofs}. 

\emph{Class selection.} 
If we want to return a single class while ensuring that it has the largest posterior predictive probability, let $\widehat c_t=\arg\max_{c\in[C]}\widehat\mu_{t,c}$ denote the current Monte Carlo prediction and let $c^\star=\arg\max_{c\in[C]}\mu_c$ denote the leading class of $\boldsymbol\mu$. We stop once the lower bound of $\widehat c_t$ exceeds every competing upper bound:
\begin{equation}
\tau_{\mathrm{cls}}
:=
\min \{
 t: L_{t,\widehat c_t} >
 \max_{c\neq\widehat c_t} U_{t,c}
\}.
\label{eq:class-stopping}
\end{equation}
If the condition is not met after a predefined $T_{\max}$ samples, set $\tau_{\mathrm{cls}} := \infty$ and abstain from prediction.
\begin{lemma}
\label{lem:class-selection}
With probability at least $1-\alpha$, whenever $\tau_{\mathrm{cls}}<\infty$, the returned class is the leading class: $\widehat c_{\tau_{\mathrm{cls}}}=c^\star$.
\end{lemma}

\emph{Predictive-distribution accuracy.}
To approximate the complete predictive vector, the coordinate-wise error can be controlled instead. One can define
\begin{equation}
 r_t
 =
 \max_c\max\{
 U_{t,c}-\widehat\mu_{t,c},
 \widehat\mu_{t,c}-L_{t,c}
 \}, \label{eq:r_t}
\end{equation}
such that stopping at
\begin{equation}
\tau_{\mathrm{dist}}
:=
\min\{t:r_t\leq\varepsilon\}
\label{eq:dist-stopping}
\end{equation}
Measuring the distance from $\widehat\mu_{t,c}$ to both $U_{t,c}$ and $L_{t,c}$ in Eq.~\eqref{eq:r_t}, rather than using the interval width, also accounts for potentially asymmetric confidence intervals.

\begin{lemma}
\label{lem:distribution-accuracy}
With probability at least $1-\alpha$, whenever $\tau_{\mathrm{dist}}<\infty$, the returned estimate satisfies $\|\widehat{\boldsymbol\mu}_{\tau_{\mathrm{dist}}}-\boldsymbol\mu\|_\infty\leq\varepsilon$.
\end{lemma}

\emph{Threshold decision.}
To determine whether any class has a mean probability of at least $q \in [0, 1]$ , we can stop either when some lower bound reaches $q$, indicating that the threshold is reached, or when all upper bounds fall below $q$,
establishing that it is not:
\begin{equation}
\tau_q
:=
\min\{
 t:
 \max_c L_{t,c}\geq q
 \ \text{or}\ 
 \max_c U_{t,c}<q
\}.
\label{eq:threshold-stopping}
\end{equation}

\begin{lemma}
\label{lem:threshold-decision}
With probability at least $1-\alpha$, whenever $\tau_q<\infty$, the rule correctly determines whether $\max_c\mu_c\geq q$.
\end{lemma}

\paragraph{A simple but powerful multivariate CS.}
The stopping rules above can be applied with any multivariate CS that provides coordinate-wise bounds. In our experiments, we use the bounding-box CS of \citet{denoodt2026tightness} because its regions are efficient to compute while producing empirically tight intervals that adapt to the variance of the observations.

At time $t$, its region is
\begin{equation}
\bbox_t
:=
\{
\mathbf m\in[0,1]^C
\;\mid\;
 m_c\in[L_{t,c},U_{t,c}],\quad c=1,\ldots,C
\},
\label{eq:bbox}
\end{equation}
where
\begin{equation}
[L_{t,c},U_{t,c}]
:=
\{
 m_c\in[0,1]
 \;\mid\;
 k_t^c(m_c)
 \leq
 C\,\!\bigl(\alpha^{-1}-\sum_{c'\neq c}w_{c'}k_{\star,t}^{c'}\bigr)
\}.
\label{eq:bbox-coordinate}
\end{equation}
Here, $k_t^c(m_c)$ denotes the scalar hedged-capital function of \citet{waudbysmith2024betting}, evaluated at the candidate mean $m_c$ using observations $y_{1,c},\ldots,y_{t,c}$, and $k_{\star,t}^c:=\min_{u\in[0,1]}k_t^c(u)$. Because $k_t^c$ is quasiconvex, its minimum and the interval endpoints can be computed with the conservative bisection procedure of \citet{denoodt2026tightness}.

We make two additional tightenings that are specific to our setting. First, we intersect the boxes over time, resulting in the endpoints
\begin{equation}
L_{t,c}:=\max_{1\le s\le t}L_{s,c},
\qquad
U_{t,c}:=\min_{1\le s\le t}U_{s,c}.
\end{equation}
Second, since $\boldsymbol\mu\in\simplex$, we can use the probability-simplex constraint to further tighten each coordinate interval:
\begin{equation}
\begin{aligned}
L_{t,c}
&:=\max{L_{t,c},\ 1-\sum_{c'\ne c}U_{t,c'}}, \quad
U_{t,c}
&:=\min{U_{t,c},\ 1-\sum_{c'\ne c}L_{t,c'}}.
\end{aligned}
\label{eq:simplex-hull}
\end{equation}
These updates exploit both past regions and the simplex constraint while preserving coverage guarantees.

\FloatBarrier
\section{Experiments}
We focus the evaluation on the \emph{class selection} stopping rule, which corresponds to reporting the most likely class. We evaluate it in simulation and on CIFAR-100 with $\alpha=.05$. If the rule has not stopped by $T_{\max}$, we return $\widehat c_{T_{\max}}$ as unresolved, without a statistical guarantee.

\paragraph{Simulated three-class problem.}
For each of 500 three-class inputs, we model stochastic predictions as
$\mathbf y_t\mid(\boldsymbol\mu,\kappa)\sim\mathrm{Dirichlet}(\kappa\boldsymbol\mu)$. For 65\% of the inputs, we use an ``easy'' setup with $\boldsymbol\mu=(.75,.25r,.25(1-r))$, where $r\sim\mathrm{Beta}(1,1)$ and $\kappa\sim\mathrm{Unif}[80,300]$. The remaining 35\% simulate ambiguous predictions with $\boldsymbol\mu=(.46,.44,.10)$ and $\kappa\sim\mathrm{Unif}[30,120]$. The true class $c^\star$ is therefore quickly separated for easy inputs, whereas ambiguous inputs may require more samples to detect it. We set $T_{\max}=400$.

\begin{figure}[t]
  \centering
  \includegraphics[width=\linewidth]{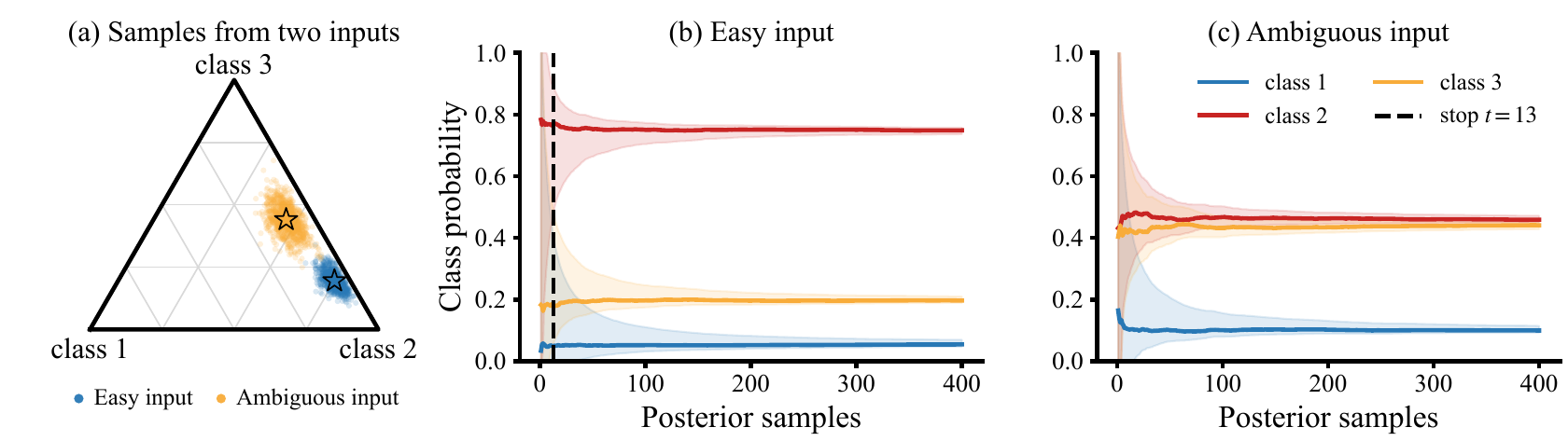}
  \caption{
\textbf{Example of the \emph{class selection} rule.} (a) True means and samples for an easy and an ambiguous input. (b) The CS detects the easy input's leading class at $t=13$. (c) The ambiguous input remains unresolved at $T_{\max}=400$.
}
  \label{fig:toy}
\end{figure}

\begin{wrapfigure}[13]{r}{.4\linewidth}
\vspace{-2em}
  \centering
  \includegraphics[width=\linewidth]{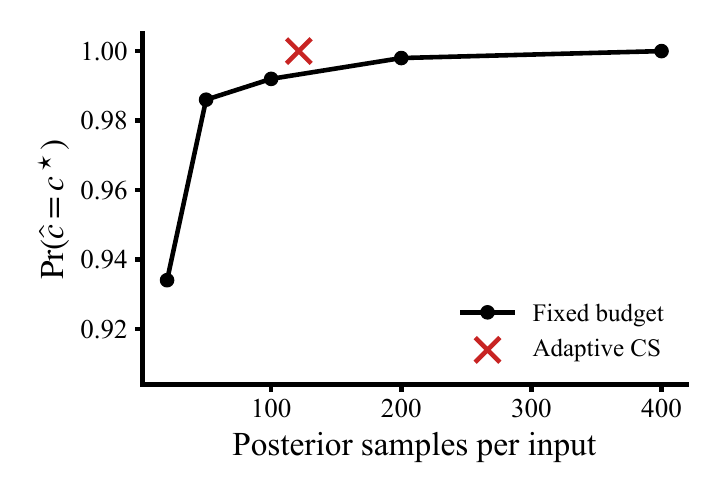}
  \caption{
  \textbf{Sample efficiency on simulated predictions.} Adaptive stopping requires fewer samples per input to match accuracy.
  }
  \label{fig:toy-accuracy}
  \vspace{-1em}
\end{wrapfigure}
Figure~\ref{fig:toy} illustrates the contrast between an easy input and a near tie. Across the simulation, the adaptive rule recovers $c^\star$ on every input, while roughly one fifth reach the cap unresolved. Figure~\ref{fig:toy-accuracy} shows that the adaptive rule spends the compute budget more wisely than a fixed budget per input.

\paragraph{CIFAR-100.} We evaluate the rule on CIFAR-100 \citep{krizhevsky2009learning} using a public ResNet-18 \citep{he2016resnet} checkpoint with 79.26\% test accuracy \citep{resnet18cifar100}. With \texttt{laplace-torch} \citep{daxberger2021laplace}, we fit a diagonal generalized Gauss-Newton Laplace approximation over \emph{all} network parameters:
\begin{equation}
 q(\boldsymbol\theta)=\mathcal N\!\left(\widehat{\boldsymbol\theta},
 \left(\operatorname{diag}(G_{\rm GGN})+\lambda I\right)^{-1}\right).
 \label{eq:laplace}
\end{equation}

We compute curvature from the 50,000 training images and select prior precision $\lambda=3\times10^5$ by predictive NLL on a disjoint 1,000-image calibration set. Each posterior sample requires a full-network forward pass. Since the exact predictive mean $\boldsymbol \mu$ in Eq.~\eqref{eq:mu_predictive_posterior} is unavailable, an independent 1,000-sample Bayesian model average (BMA) defines the reference class $c^{\mathrm{ref}}$. We evaluate 1,000 test images with $T_{\max}=200$.

\WFclear
\begin{wrapfigure}[20]{r}{.4\linewidth}
  \vspace{-1.7em}
  \centering
  \includegraphics[width=\linewidth]{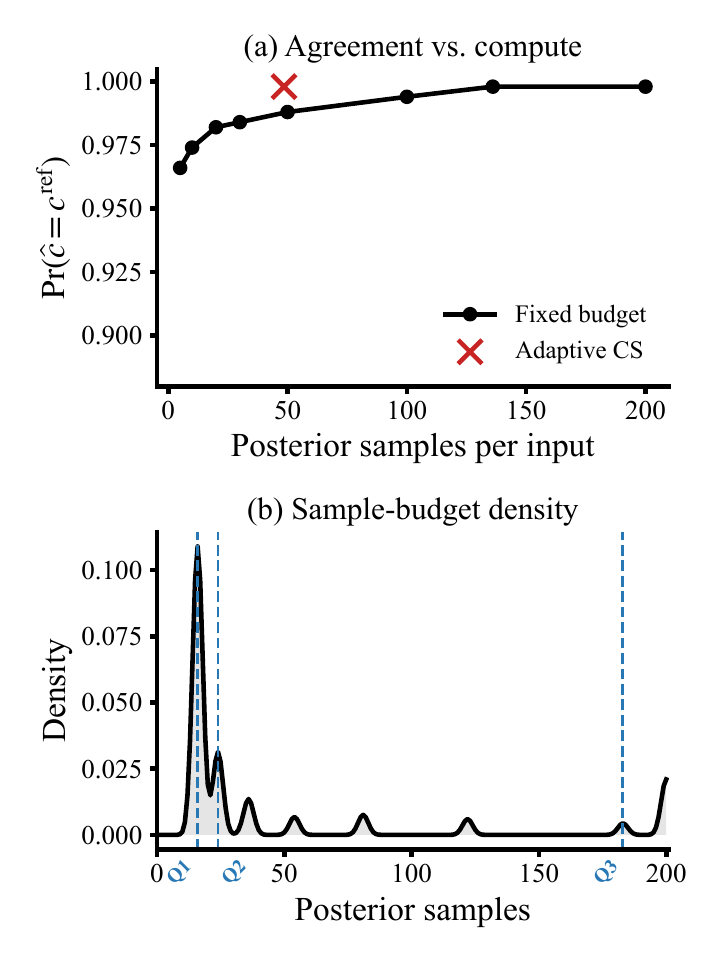}
  \vspace{-2em}
  \caption{\textbf{Sample efficiency on CIFAR-100.} Agreement with the reference BMA (top) and the allocated-sample distribution with Q1--Q3 quartiles (bottom).}
  \label{fig:cifar}
  \vspace{-1em}
\end{wrapfigure}

We inspect the stopping rule only at $t\in\{16,\allowbreak 24,\allowbreak 36,\allowbreak 54,\allowbreak 81,\allowbreak 122,\allowbreak 183,\allowbreak 200\}$, although every draw updates the CS. This reduces the computational overhead of constructing $\bbox$, which is otherwise costly relative to a single forward pass.    


Figure~\ref{fig:cifar} shows 99.8\% agreement with the reference using 48.6 samples on average. The rule stops for 89.9\% of inputs, and every resolved prediction agrees with the BMA. The vertical lines mark the Q1, Q2, and Q3 quartiles at 16, 24, and 183 samples, respectively, indicating that the majority of inputs require only a small number of samples. The smallest fixed budget reaching the same 99.8\% agreement is $T^\star=136$ samples per input.


\FloatBarrier
\paragraph{Does sample efficiency translate into time efficiency?}
We benchmark mean latency of making predictions on an NVIDIA A100, including posterior sampling, forward passes, and CS computation.
Table~\ref{tab:runtime} shows that adaptive inference averages 0.253s per input, including 0.055s of CS overhead. It is nevertheless 3.21$\times$ faster than 200 posterior samples and 2.24$\times$ faster than the performance-matched $T^\star=136$. Fixed 30 is faster but reaches only 98.4\% agreement.

\begin{table}
\centering
\caption{
	Mean latency per CIFAR-100 input, measured over 1,000 inputs. Adaptive sampling uses $\bar T=48.6$ posterior samples on average and achieves the same 99.8\% agreement as the smallest matching fixed Monte Carlo budget, $T^\star=136$, while being $2.24\times$ faster.
	}
\label{tab:runtime}
\small
\begin{tabular}{lccccc}
\toprule
Method & Inference (s) & CS overhead (s) & Total (s) & Adaptive vs. row & $\Pr(\hat c = c^\mathrm{ref})$ \\
\midrule
Fixed 30 & 0.125& 0 & 0.125 & 2.03$\times$ slower & 98.4\% \\
Fixed 200 & 0.813 & 0 & 0.813 & 3.21$\times$ faster & 99.8\% \\
\midrule
Fixed $T^\star=136$ & 0.566 & 0 & 0.566 & 2.24$\times$ faster & 99.8\% \\
Adaptive ($\bar T=48.6$) & 0.198 & 0.055 & \textbf{0.253} & -- & 99.8\% \\
\bottomrule
\end{tabular}
\end{table}

\section{Conclusion}
We introduce an anytime-valid approach for dynamically determining the number of Monte Carlo samples needed for each input in Bayesian neural network inference. Using confidence sequences, sampling stops once the prediction can be made with a statistical guarantee. This saves samples on easy inputs while using more computation for ambiguous ones, reducing overall inference time without sacrificing reliable decisions.

\section*{Acknowledgments and Disclosure of Funding}
We thank Joaquin Vanschoren for their valuable input and discussions. This work was funded by the MOSAIC project under Grant Agreement No. 101194414. 

Funded by the European Union. Views and opinions expressed are however those of the author(s) only and do not necessarily reflect those of the European Union or the European Commission. Neither the European Union nor the European Commission can be held responsible for them.

\bibliographystyle{plainnat}
\bibliography{biblio}

\appendix
\section{Proofs of the stopping-rule guarantees}
\label{app:stopping-proofs}

\paragraph{Proof of Lemma~\ref{lem:class-selection}.}
\begin{proof}
Denote $\widehat c=\widehat c_{\tau_{\mathrm{cls}}}$. Then, on the simultaneous coverage event in Eq.~\eqref{eq:cs-guarantee}, for every $c\neq\widehat c$, the stopping condition gives
$$
\mu_{\widehat c}
	\geq L_{\tau_{\mathrm{cls}},\widehat c}
	> U_{\tau_{\mathrm{cls}},c}
	\geq \mu_c.
    $$
Thus $\widehat c_t$ is the leading class. The coverage event has probability at least $1-\alpha$.
\end{proof}

\paragraph{Proof of Lemma~\ref{lem:distribution-accuracy}.}
\begin{proof}
On the simultaneous coverage event, $\mu_c\in[L_{t,c},U_{t,c}]$ for every $t$ and $c$. Hence $|\widehat\mu_{t,c}-\mu_c|\leq\max\{U_{t,c}-\widehat\mu_{t,c},\widehat\mu_{t,c}-L_{t,c}\}\leq r_t$. At $t=\tau_{\mathrm{dist}}$, this is at most $\varepsilon$ for every coordinate.
\end{proof}

\paragraph{Proof of Lemma~\ref{lem:threshold-decision}.}
\begin{proof}
On the simultaneous coverage event, $\max_cL_{t,c}\geq q$ implies that $\mu_c\geq q$ for at least one class, whereas $\max_cU_{t,c}<q$ implies that $\mu_c<q$ for every class. Thus either stopping condition guarantees the corresponding decision.
\end{proof}

\end{document}